\documentclass{article}

\usepackage[preprint]{neurips_2026}

\usepackage{amssymb,amsmath,amsthm,amsfonts}
\usepackage{wrapfig}
\usepackage{placeins}
\usepackage{needspace}

\usepackage{microtype}
\usepackage{graphicx}
\usepackage{subcaption}
\usepackage{booktabs} %

\newcommand{\appref}[1]{\hyperref[#1]{Appendix~\ref*{#1}}}
\newcommand{\LLL}{\mathcal{L}}

\usepackage{amsmath,amssymb,amsthm}
\usepackage{bm}
\usepackage{booktabs}
\usepackage{enumitem}

\usepackage{mathtools}

\theoremstyle{plain}
\newtheorem{theorem}{Theorem}[section]

\newtheorem{lemma}[theorem]{Lemma}

\theoremstyle{definition}

\theoremstyle{remark}

\usepackage[textsize=tiny]{todonotes}
\usepackage{siunitx}
\usepackage{ctable}

\usepackage{color, colortbl}

\usepackage{multicol}

\usepackage{xcolor}
\usepackage{makecell}
\usepackage{colortbl}

\definecolor{mydarkblue}{rgb}{0,0.08,0.45}
\usepackage[colorlinks=true,
    linkcolor=mydarkblue,
    citecolor=mydarkblue,
    filecolor=mydarkblue,
    urlcolor=mydarkblue,
    pdfview=FitH]{hyperref}

\renewcommand{\sectionautorefname}{Section}
\renewcommand{\subsectionautorefname}{Section}

\usepackage{bbm}
\usepackage[utf8]{inputenc} %
\usepackage[T1]{fontenc}    %
\usepackage{hyperref}       %
\usepackage{url}            %
\usepackage{booktabs}       %
\usepackage{amsfonts}       %
\usepackage{nicefrac}       %
\usepackage{microtype}      %
\usepackage{xcolor}         %
\usepackage{tcolorbox}

\usepackage{minitoc}

\title{LeRoPE: Learnable RoPE Frequencies Improve Language Modeling}

\author{%
  Petros Karypis\thanks{Primary authors. Correspondence: pkarypis@ucsd.edu} \\
  UC San Diego \\
  \And 
  Sean O'Brien$^*$ \\
  UC San Diego \\
  \And
  Shreyas Kadekodi \\
  UC San Diego \\
  \And
  Rui Zhu \\
  Independent Researcher \\
  \AND
  Julian McAuley \\
  UC San Diego \\
}

\begin{document}
\doparttoc

\maketitle
\begin{abstract}
Rotary Positional Encodings (RoPE) are currently the most popular positional encodings used in modern language models. 
RoPE rotates two-dimensional chunks of query and key vectors, operating as a function of their relative positional offset.
The position-wise rates of rotation in RoPE typically follow a geometric sequence specified by a fixed base-frequency hyperparameter.
Prior work has improved performance by either increasing this parameter to slow rotation or by applying RoPE to only a subset of QK dimensions.
In this work we modify RoPE by learning a scalar per frequency, treating frequencies as learnable parameters rather than hyperparameters.
We validate \textbf{Le}arned \textbf{RoPE} by training a ladder of language models from scratch, ranging from 52M to 2.5B parameters. We observe and analyze the emergence of a high-norm, positional LeRoPE band. LeRoPE consistently outperforms RoPE and partial RoPE across all scales, with RoPE requiring 3.4\% more compute (FLOPs) to match LeRoPE at the largest scale.
\end{abstract}
\faketableofcontents

\section{Introduction}
\label{sec:intro}

Positional encodings are typically included in Transformer-based models \citep{Vaswani2017AttentionIA} because the self-attention operation on its own does not directly carry positional signal. %
Rotary Positional Encodings (RoPE)~\citep{Su2021RoFormerET} are used by a number of popular language models~\citep{Touvron2023LLaMAOA,
Chowdhery2022PaLMSL, Jiang2023Mistral7, Kamath2025Gemma3T}.
RoPE splits the
query and key vectors into two-dimensional pairs, or \emph{frequency bands},
and rotates each pair by an angle proportional to token position. 
The effect of these rotations on the query--key inner product depends only on the relative offset between tokens. 
Each band rotates at its own fixed rate, following a geometric sequence set by a base hyperparameter: high-frequency
bands complete many rotations within a context, while the lowest-frequency bands rotate minimally. 

The choice of these frequencies matters. Recent work finds that the slowest bands carry semantic rather than positional information~\citep{Barbero2024RoundAR}. 
This positional invariance allows content matches to hold at any separation, making it beneficial to remove rotation from a subset of bands entirely, or to adjust the base, which sets how many bands rotate appreciably within the context
window~\citep{Barbero2024RoundAR,Men2024BaseOR,oka2026frequency}.

Motivated by the above, and noting that the RoPE operation is differentiable with respect to its frequency, we propose learning the frequencies.
Learned RoPE (LeRoPE) adds one learned scalar per frequency band, 32 parameters in total for each of our models.
Training on a ladder of models from 52M to 2.5B parameters, LeRoPE outperforms RoPE at every scale. At the largest scale, RoPE requires $3.4\%$ more compute to match the performance of LeRoPE (\autoref{fig:lerope_cm}).

Analyzing models trained with LeRoPE, we observe a consistent frequency profile across scales and random seeds, as well as the emergence of a single band with a large, strongly positional contribution to the attention logits, which we term the ``dominant positional band.''

\begin{figure}[tb]
  \centering
  \includegraphics[width=\linewidth]{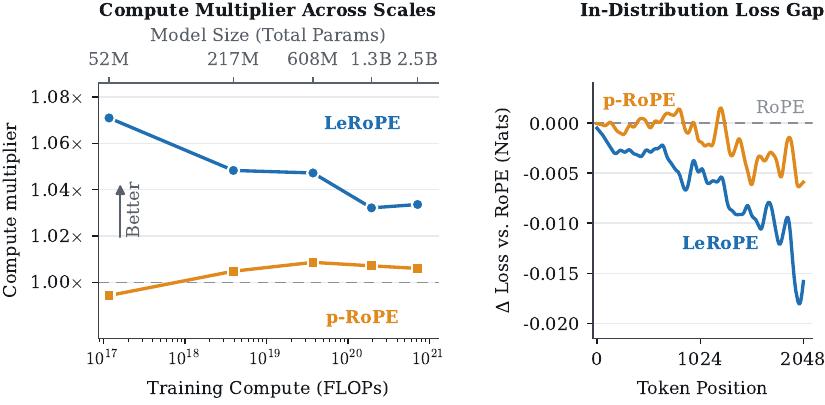}
  \caption{\textbf{a}: Compute multipliers vs.\ RoPE. The compute multiplier is the factor by which the baseline's training compute must increase to match the variant's loss. The multiplier declines from ${\sim}1.07$ at the smallest scale and plateaus near $1.03$--$1.04$, a small but persistent advantage. This means RoPE needs $3.4\%$ more compute to reach the loss LeRoPE achieves at 2.5B.
  All runs use a single seed; learning rates are swept per method for models below 1.34B (\S\ref{ssec:ladder}). \textbf{b}: Average loss gap per token position of $p$-RoPE and LeRoPE vs.\ RoPE, in-distribution (smoothed via Gaussian kernel, $\sigma = 30$).}
  \label{fig:lerope_cm}
\end{figure}

\section{Preliminaries}
\label{sec:prelims}

\subsection{Positional Embeddings}
\label{ssec:related_works}

To break the positional symmetry of the attention mechanism, \cite{Vaswani2017AttentionIA} add fixed sinusoidal positional embeddings directly to the input residual stream. Later models instead incorporated learned absolute positional embeddings ~\citep{Radford2019LanguageMA,Devlin2019BERTPO,Liu2019RoBERTaAR,Zhang2022OPTOP,Dosovitskiy2020AnII}.

\cite{shaw-etal-2018-self} proposed \textit{relative} position representations, which depend only on the relative offset between tokens; a simplified, scalar version of learned relative offsets was adopted by T5 \citep{Raffel2019ExploringTL}.
Rotary positional embeddings, proposed by \cite{Su2021RoFormerET}, also depend only on relative position, and have become the standard choice for modern language models.

Since RoPE, a number of positional encodings have been proposed. 
ALiBi \citep{Press2021TrainST} improved out-of-distribution length extrapolation with fixed relative positional biases.
CoPE provides the model with the capacity to produce relative positional embeddings from learned, content-dependent features \citep{Golovneva2024ContextualPE}; PaTH also proposes a data-dependent relative transform based on accumulating Householder matrices \citep{Yang2025PaTHAP}.

Some methods generalize to a broader class of functions, of which RoPE is a member. For instance, LieRE learns skew-symmetric matrices that can rotate along planes other than the standard RoPE plane \citep{Ostmeier2024LieRELR}.
\cite{oka2026probing} generalize RoPE to a broader class of wavelet-function-based positional encodings, finding improved performance from Ricker-based wavelets.

\subsection{RoPE}
\label{ssec:rope}

\cite{Su2021RoFormerET} propose \textbf{Ro}tary \textbf{P}osition \textbf{E}mbedding (\textbf{RoPE}), a relative positional embedding method that forms the basis of our study. RoPE rotates key and query vectors according to their position, modifying their inner product according only to their relative offset.

Concretely, consider a single attention head with per-head dimension $d_h$. The dimensions of this head can be chunked into $d_h/2$  frequency pairs, indexed by $m = 0, \ldots, d_h/2-1$. Standard RoPE assigns each pair a frequency from a geometric series:
\begin{equation}\label{eq:theta}
    \theta_m = b^{-2m/d_h}, \quad m = 0, \ldots, d_h/2-1,
\end{equation}
where $b > 1$ is the base (typically $b = 10{,}000$). Following the notation of~\cite{Barbero2024RoundAR} we denote $\rho(\theta_m)$ as the matrix form of the rotational transformation of angle $\theta_m$:
\begin{equation}
    \rho(\theta_m) = 
    \begin{bmatrix}
    \cos(\theta_m) & -\sin(\theta_m) \\
    \sin(\theta_m) & \cos(\theta_m)
    \end{bmatrix}
\end{equation}

\newcommand{\Rs}{\textbf{R}_s}
\newcommand{\Rt}{\textbf{R}_t}
RoPE constructs a block-diagonal matrix $\Rs = {\oplus}_{m=0...d_h /2-1} \rho (\theta^{(m)})^s$ where $s$ is the position of a given token.
It then applies $\Rs$ to each key and query vector; that is $\tilde{q}_s = \Rs q_s$ and $\tilde{k}_s = \Rs k_s$.

The rotation-matrix properties $\rho(\theta)\rho(\theta') = \rho(\theta + \theta')$ as well as $\rho(\theta)^\intercal = \rho(-\theta)$, in conjunction with the block-diagonal structure of $\Rs$ gives the relation $\Rs^\intercal  \Rt = \textbf{R}_{t - s}$. The dot product of a RoPE-rotated query-key pair $\tilde{q}_s, \tilde{k}_t$ decomposes into:

\begin{align}\label{eq:rope_band_decomp}
    \tilde{q}_s^\intercal \tilde{k}_t^{\vphantom{\intercal}} = (\Rs q_s)^\intercal (\Rt k_t) = q_s^\intercal \textbf{R}_{t-s}k_t = \sum_{m=0...d_h/2-1} \left( q_s^{(m)} \right)^T \rho(\theta_m)^{t-s} k_t^{(m)} 
\end{align}

where $q_s^{(m)}$ is the 2 dimensional chunk (or band) of the query vector $q_s^{[2m,2m+1]}$ (the same notation holds for the key vectors).  Thus, the RoPE operation depends only on the relative offset $t - s $ and not on the absolute positions $s$ and $t$.

\subsection{Frequency Bands and Partial RoPE}
\label{ssec:prope}

Each index $m$ in Eq.~\eqref{eq:theta} defines a \emph{frequency band}: a two-dimensional sub-vector 
rotated at angular frequency $\theta_m$, completing a full
period every $2\pi/\theta_m$ tokens. Because the frequencies are geometrically
spaced, these periods span several orders of magnitude, and the slowest bands
typically do not complete a single period within the training context.

Trained models do not use this spectrum uniformly.
\citet{Barbero2024RoundAR} show that RoPE models concentrate query and key
norm in the lowest-frequency bands and argue that these bands, which rotate
negligibly over typical offsets, transport semantic rather than positional
information. Since even a slow band eventually drifts over long ranges, they
propose Partial RoPE ($p$-RoPE), which removes rotation from the slowest bands
entirely:
\begin{equation}\label{eq:prope}
    \theta_m =
    \begin{cases}
        b^{-2m/d_h} & m < p \cdot d_h/2, \\
        0           & \text{otherwise},
    \end{cases}
\end{equation}
where $p$ is the fraction of bands kept. Setting $p{=}1$ recovers RoPE, while
$p{=}0$ recovers NoPE — omitting positional encoding altogether — which itself
improves length generalization \citep{Kazemnejad2023TheIO}. Adjusting
the base $b$ acts on the same lever continuously: raising it slows every band
rather than zeroing the slowest, and is an effective tool for context
extension \citep{Liu2023ScalingLO, Men2024BaseOR}, adopted at scale in recent
open models \citep{Kamath2025Gemma3T}.

\citet{oka2026frequency} expand this analysis and identify a contiguous \emph{high-norm
band} of RoPE dimensions that emerges early in pre-training, with its location
determined jointly by the base $b$ and the training length. Lower frequency bands can be replaced
with NoPE at inference at little cost, confirming that the slowest frequencies
are only weakly used for position. All of these methods, however,
fix the band allocation by hand before training.

\section{LeRoPE}

\subsection{Extending  RoPE to Learned Frequencies}
\label{ssec:lerope}

\begin{tcolorbox}[boxsep=0mm,left=2.5mm,right=2.5mm]
LeRoPE augments RoPE with one learned scalar $\alpha_m$ per frequency band,
rescaling each frequency:
\begin{equation}\label{eq:lerope}
    \hat\theta^{(m)} = e^{\alpha_m}\theta^{(m)},
    \qquad
    \mathbf{R}^{\mathrm{LeRoPE}}_s
    = \bigoplus_{m=0}^{d_h/2-1} \rho\big(\hat\theta^{(m)}\big)^{s}.
\end{equation}
\end{tcolorbox}

We parametrize frequency scalars in logarithmic space, so that steps in $\alpha_m$ make multiplicative changes to $\hat\theta^{(m)}$. 
This yields better conditioning when scaling down frequencies and avoids undesirable gradient dynamics near frequencies of zero. We initialize $\alpha_m = 0$, making our method equivalent to RoPE on initialization. Additionally, we omit weight decay on LeRoPE parameters, and apply gradient clipping to them independently of the rest of the model's parameters.

All layers and heads share the same set of learned frequencies; thus, with $d_h$ as the head dimension, LeRoPE adds $d_h/2$ total parameters to the model.
The parametrization also admits finer granularity (e.g., per-layer or per-head), but we find that sharing $\alpha_m$ values across all layers and heads performs consistently best.

\subsection{Gradient of a Frequency}
\label{ssec:frequency_gradient}

Each learned frequency receives its
gradient through the per-band contributions of
Eq.~\eqref{eq:rope_band_decomp}. Writing the band-$m$ chunks in polar form
and differentiating,

\begin{align}
\label{eq:freq_grad}
  \frac{\partial \mathcal{L}}{\partial \hat\theta^{(m)}}
  &= -\sum_{s}\sum_{t < s} (s-t)\,
    \big\|q_s^{(m)}\big\|\,\big\|k_t^{(m)}\big\|\,
    \sin\!\big(\phi^{(m)}_{st} + (s-t)\,\hat\theta^{(m)}\big)\,
    \frac{\partial \mathcal{L}}{\partial z_{st}}
\end{align}

where $\phi^{(m)}_{st}$ is the angle from the key to the query chunk and $z_{st}$ the
attention logit.

We can express the gradient update as the derivative of a surrogate loss $P(\theta)$. Let $\alpha_{st}$ be the post-softmax attention weight from query $q_s$ to key $k_t$; let $v_t$ be the value vector at position $t$ and $o_s = \sum_{t < s} \alpha_{st} v_t$ be the attention-aggregated output vector, before projection for simplicity. Further, let $g_s \triangleq \partial \LLL / \partial o_s$ be the downstream gradient and $L$ the maximum context length.

With these quantities, we can define a spectral function of offset $C_m(\delta)$. The direction of each term is dictated by the relative QK angle $\phi_{st}^{(m)}$, and its scale is the product of QK norms, post-softmax probability, and the downstream gradient term.

\begin{align*}
C_m(\delta) &\triangleq \sum_{\substack{s,t \\s-t=\delta}} \alpha_{st}   \|q_s^{(m)} \| \|k_t^{(m)}\|
    e^{i\phi_{st}^{(m)}}
    (v_t - o_s)^T g_s = R_m (\delta) e^{i\Phi_m(\delta)} \\
    P_m(\theta) &= \text{Re}\left(
    \sum_{\delta=1}^{L-1} C_m(\delta) e^{i \delta \theta}
    \right) = \sum_{\delta=1}^{L-1} R_m(\delta) \cos \left(
    \Phi_m(\delta) + \delta \theta
    \right)
\end{align*}

\begin{lemma}
\label{lemma:dtft_gradient}
Define $P_m(\theta)$ as above: the real component of the discrete-time Fourier Transform of the offset-aggregated gradient signal $C_m(\delta)$ with frequency $\theta$.

Then $\frac{\partial \LLL}{\partial \theta^{(m)}} = P_m'\left(\theta^{(m)}\right)$.
\end{lemma}

\begin{proof}
    See Lemma~\ref{lem:dtft_gradient}.
\end{proof}

This provides one perspective on how $\theta^{(m)}$ is updated.
Each offset $\delta$ is mapped to a complex phasor $C_m(\delta)$, which aggregates the downstream gradients weighted by QK norms, un-rotated QK relative angles, and the post-softmax attention weights $\alpha_{st}$.

Then, $\theta^{(m)}$ is updated towards the local minimum of the surrogate function $P_m(\theta)$, which is the real component of the discrete-time Fourier Transform of the series, evaluated at frequency $\theta$.
The effect of each relative offset on this gradient can be characterized by visualizing the per-offset phasors $C_m(\delta)$: we plot these for each band in \appref{sec:phasor_plots} to illustrate the per-offset geometries; the phasors largely cancel about the origin rather than cohering.

\section{Experimental Setup}
\label{sec:exp_setup}

\begin{table}[tb]
  \centering
  \caption{Validation perplexity on C4 across our scaling ladder (lower is better; best per column in bold). LeRoPE improves
  over the RoPE baseline at every scale; $p$-RoPE yields a smaller, consistent
  improvement at all but the smallest scale.}
  \label{tab:ppl}
  \begin{tabular}{lccccc}
    \toprule
    Method & 52M & 217M & 608M & 1.34B & 2.52B \\
    \midrule
    RoPE   & 30.2450 & 18.2377 & 13.8795 & 11.5747 & 10.1720 \\
    $p$-RoPE & 30.2695 & 18.2269 & 13.8663 & 11.5665 & 10.1659 \\
    \midrule
    LeRoPE & \textbf{29.9479} & \textbf{18.1335} & \textbf{13.8093} & \textbf{11.5384} & \textbf{10.1387} \\
    \bottomrule
  \end{tabular}
\end{table}

\subsection{Scaling Ladder}
\label{ssec:ladder}

To evaluate our proposed method we train a family of decoder-only transformers ranging in size from $52\text{M}$ to $2.5\text{B}$ parameters under a compute-optimal regime. Following previous works~\citep{Ferbach2026LogarithmictimeSF, Charles2025CommunicationEfficientLM} we co-scale layers and head count keeping the head dimension fixed at 64 so all models have the same number of learned frequency scales. We adopt the Chinchilla token count of $D = 20N$ where $N$ is the total number of model parameters. We use C4 \citep{Raffel2019ExploringTL} as our pre-training corpus and a sequence length of 2048 with intra-document masking. During evaluation we do not pack documents. An exhaustive overview of our architecture, hyper-parameters and differences from existing works can be found in Appendix~\ref{ssec:arch}.

We compare against RoPE with the standard base $b=10{,}000$ and $p$-RoPE with $p=0.75$ following~\citet{Barbero2024RoundAR}. For every model below 1.34B parameters we sweep the learning rate over a geometric grid, $\eta \in \{2^{-i/2} : i \in \mathbb{Z}\}$, until the best value is bracketed by a tested value on either side. We run this sweep independently for each of the three positional-encoding methods, and find that all three share the same optimal learning rate at every scale; comparisons at a given scale therefore use identical hyperparameters. For the two largest models we extrapolate the best grid index from the smaller scales.

\subsection{Compute Multipliers}
\label{ssec:compute_multipliers}

To evaluate the benefit of a proposed architecture change, an important quantity to measure is the amount of compute saved relative to the baseline in order to reach the same performance. We report this through compute multipliers: a compute multiplier of $1.x$ for LeRoPE corresponds to needing $x\%$ more compute to reach the same loss with RoPE.

We measure the training compute of each model as $C = 6 ND$ FLOPs, where $N$ is its number of non-embedding parameters. 
For each method we form a loss--compute curve by piecewise-linear interpolation of the
final validation losses of its compute-optimal models in $(\log C, L)$ space and
extrapolating beyond the smallest and largest scales with the slope of the
nearest segment. Writing $L^{\mathrm{RoPE}}$ and $L^{\mathrm{var}}$ for the
baseline and variant curves, the compute multiplier of a variant trained at
compute $C^{\mathrm{var}}$ is
\begin{equation}
  \mathrm{CM} \;=\; \frac{C^{\mathrm{RoPE}}}{C^{\mathrm{var}}},
  \qquad\text{where}\qquad
  L^{\mathrm{RoPE}}\!\big(C^{\mathrm{RoPE}}\big)
  = L^{\mathrm{var}}\!\big(C^{\mathrm{var}}\big),
  \label{eq:cm}
\end{equation}
i.e.\ the factor by which RoPE's compute must increase to reach the variant's
loss, obtained by inverting the interpolated $L^{\mathrm{RoPE}}$.

Our scaling ladder is sparser than that of~\cite{Ferbach2026LogarithmictimeSF},
with adjacent scales differing by a decade or more of compute. Linear interpolation
across these gaps overestimates the multiplier for the convex power-law loss
curve, so we invert \autoref{eq:cm} in $\log C$ rather than $C$ to reduce the overestimation.
We include further discussion on this choice as well as a comparison with the more traditional \textit{efficiency gain} fitted power-law inversion method~\citep{mai-thinking} in \appref{aapp:cm_calc}.

\section{Results}
\label{sec:results}

\begin{table}[t]
  \centering
  \caption{Validation perplexity and percentage of LeRoPE gain captured by each method on a 217M-parameter model. Fixed LeRoPE uses frozen frequencies learned by LeRoPE on a separate run at the same scale.}
  \label{tab:ppl_warmstart}
  \begin{tabular}{lcc}
    \toprule
    Method & Perplexity & \% Gain Captured \\
    \midrule
    RoPE   & 18.2377 & 0\% \\
    $p$-RoPE & 18.2269 & 10.4\% \\
    Fixed LeRoPE & 18.1714 & 63.6\% \\
    LeRoPE & \textbf{18.1335} & 100\% \\
    \bottomrule
  \end{tabular}
\end{table}

\needspace{15\baselineskip}
\begin{wrapfigure}[16]{r}{0.4\textwidth}
  \centering
  \includegraphics[width=0.38\textwidth]{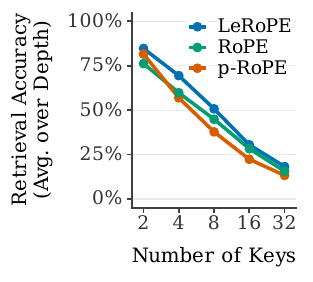}
  \caption{Multi-key Needle-in-a-Haystack accuracy as a function of the number of distractor keys evaluated on sequences in-distribution.}
  \label{fig:in_dist_niah}
  \vspace{-1em}
\end{wrapfigure}

\subsection{Compute Multipliers}

We report the compute multipliers for LeRoPE and $p$-RoPE in \autoref{fig:lerope_cm}(a), and the corresponding validation perplexity numbers in \autoref{tab:ppl}. Across all scales LeRoPE realizes a small but consistent gain, corresponding to RoPE requiring $3.4\%$ more compute to match it at the largest scale. We also include the loss delta relative to RoPE as a function of position in \autoref{fig:lerope_cm} (b). LeRoPE's gain over RoPE grows with context length, suggesting it makes better use of longer context.

To bound run-to-run variability, we repeat the 217M RoPE and LeRoPE runs with three random seeds. The method gap of $0.0066$ nats is $3$ to $7$ times the per-method seed standard deviation; LeRoPE improves on RoPE under every seed, with the worst LeRoPE run outperforming the best RoPE run by $0.0034$ nats. The exact numbers as well as additional ablations, including per-head per-layer experiments, can be found in \autoref{app:ablations}.

\subsection{In-Distribution Needle-in-a-Haystack}
\label{ssec:niah}

Needle in a Haystack (NIAH) is a standard probe of context utilization and retrieval: given a key, the model must retrieve its associated value from a context of filler text and, optionally, distractor key--value pairs. Examples are generated synthetically; we vary the depth of the gold pair (the ``needle'') within the context, using filler text from the original NIAH repository.\footnote{\url{https://github.com/gkamradt/needle-in-a-haystack}}

We evaluate the 2.5B models with 1 to 31 distractor pairs across all three methods placing the needle at depths $\{0, 0.1, \ldots, 1.0\} \times L_{\text{train}}$.
\autoref{fig:in_dist_niah} plots retrieval accuracy against distractor count, averaged over depths: LeRoPE outperforms RoPE and $p$-RoPE at every shared count, and continues to outperform RoPE at 31 distractor pairs, the hardest setting evaluated.
All evaluations here stay within $L_{\text{train}}$, as this section targets in-distribution behavior; we report only the multi-key variant because
single-needle retrieval is saturated for all models in-distribution, and we defer length extrapolation to \autoref{sec:extrap}.

\begin{tcolorbox}[boxsep=0mm,left=2.5mm,right=2.5mm]
\textbf{Section Summary:} {\em LeRoPE is $3.4\%$ more compute efficient than RoPE, leading to a final model with better in-distribution performance for a fixed compute budget.}   
\end{tcolorbox}

\section{Analysis: Learned Frequency Patterns}
\label{sec:learned_freqs}

We now study what LeRoPE learns.
We examine the learned frequencies directly and identify a single band that carries most of the positional signal.

\begin{figure}[h]
  \centering
  \begin{subfigure}[t]{0.40\linewidth}
    \centering
    \includegraphics[width=\linewidth]{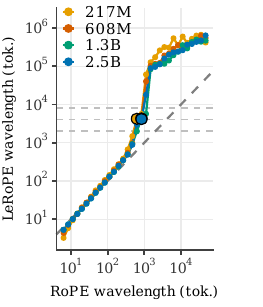}
    \label{fig:param_ladder}
  \end{subfigure}\hspace{-0.1025\linewidth}%
  \begin{subfigure}[t]{0.40\linewidth}
    \centering
    \includegraphics[width=\linewidth]{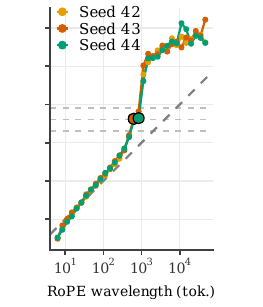}
    \label{fig:param_seeds}
  \end{subfigure}\hspace{-0.1025\linewidth}%
  \begin{subfigure}[t]{0.40\linewidth}
    \centering
    \includegraphics[width=\linewidth]{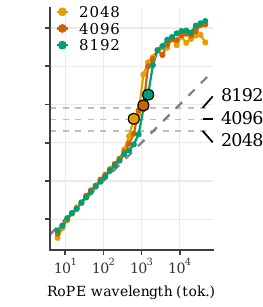}
    \label{fig:param_lengths}
  \end{subfigure}

  \caption{\textbf{(a)}: Variation in learned wavelengths across model sizes. \textbf{(b)}: Variation of learned wavelengths across 3 random seeds. \textbf{(c)}: Variation across different training lengths. Dashed horizontal lines mark shared 2048-, 4096-, and 8192-token context-length references; circled dots mark the dominant positional band; labels appear in the right panel. All plots show LeRoPE learns very similar wavelength profiles regardless of initialization and model size. The center and right plots are for the 217M models.}
  \label{fig:rope_lerope_parametric}
\end{figure}

\textbf{Learned Frequencies.} \autoref{fig:rope_lerope_parametric} plots each learned band's wavelength against RoPE's counterpart: low-wavelength bands closely track RoPE ($y=x$), mid-wavelength bands begin to diverge, and high-wavelength bands are pushed to a much larger, near-constant wavelength -- i.e.\ their frequency is driven toward zero.
High-frequency (low-wavelength) bands rotate slightly faster, with the effect more pronounced for smaller scales.
Mid-frequency bands begin to slow compared to RoPE frequencies, and low frequencies are driven down even closer to zero.
This pattern is consistent across all model scales and seeds we test (\autoref{fig:rope_lerope_parametric}a,b).

The point at which frequencies are suppressed is linked to the sequence lengths seen during training (\autoref{fig:rope_lerope_parametric}c): models trained on longer contexts diverge from the RoPE frequencies at lower frequencies than their short-context counterparts.

\textbf{Identifying Dominant Positional Bands.} \autoref{eq:rope_band_decomp} provides a per-band decomposition of the attention logit contributions.
By plotting the average contribution per relative distance, we find (\autoref{fig:hero_band}) that one band dominates 
in a nearly purely positional pattern for models with LeRoPE. 

The dominant band's wavelength tracks the training length across six runs: four model scales at $L_{\text{train}}=2048$, plus 217M runs at $L_{\text{train}}\in\{4096,8192\}$. The dominant band's wavelength is $\approx 2.205\,L_{\text{train}}$ (exact wavelengths can be in ~\autoref{tab:dom_band_stats}).
Fixed-frequency RoPE models trained at the same scale do not exhibit a dominant band. Instead, their positional contributions are distributed across several bands at far smaller magnitudes, %
although, as noted by \citet{oka2026probing}, they do contain strongly oscillatory channels.

\textbf{Learning With Fixed LeRoPE Frequencies.} Next, we ask whether the improvements from LeRoPE emerge from jointly training frequencies with the rest of the network, or solely from the discovery of a more optimal set of frequencies. To distinguish these, we train a 217M model with its frequencies frozen to the values learned by an independent LeRoPE run at the same scale (\autoref{tab:ppl_warmstart}).

Training with fixed LeRoPE frequencies yields 63\% of the improvement of full LeRoPE, suggesting that joint training dynamics and the final frequency set are both partially responsible for performance gains. Under the same settings, $p$-RoPE captures $10.4\%$ of the performance gains of LeRoPE. We can interpret $p$-RoPE as a coarse approximation of final LeRoPE frequencies, preserving faster frequencies while suppressing slower ones.

\begin{figure}[h]
  \centering
  \includegraphics[width=0.82\linewidth]{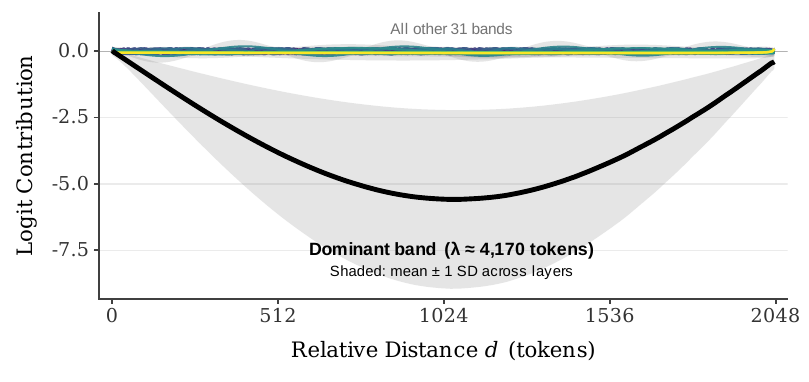}
  \caption{Band-level contributions to the logit score as a function of relative distance $d = s - t$ (in tokens), in a 2.5B LeRoPE model; the shaded band is $\pm 1$ standard deviation across the model's layers. LeRoPE produces a dominant positional band that penalizes logit scores according to their relative position; despite this, LeRoPE outperforms standard RoPE in perplexity and NIAH evaluations at all in-distribution offsets.}
  \label{fig:hero_band}
\end{figure}

\begin{tcolorbox}[boxsep=0mm,left=2.5mm,right=2.5mm]
\textbf{Section Summary:} {\em LeRoPE learns the same frequency profile
across seeds and scales: frequencies too slow for the training window are
suppressed, and a single dominant band emerges at $\lambda \approx
2.2\,L_{\text{train}}$. 
Training with the learned frequencies frozen captures $63.6\%$ of the gain.}
\end{tcolorbox}

\section{Extrapolation}
\label{sec:extrap}

While the main focus of our work is on in-distribution performance, understanding how LeRoPE extrapolates is of practical importance. RoPE-based models degrade sharply when operating on sequence lengths longer than $L_{train}$ with unmodified positions~\citep{Press2021TrainST}. The standard solution is to rescale frequencies back in-distribution at inference via positional interpolation (PI)~\citep{Chen2023ExtendingCW} or NTK-by-parts scaling with attention temperature~\citep{Peng2023YaRNEC}.

\textbf{Naive Extrapolation.} With unmodified positions, LeRoPE degrades more sharply than RoPE and $p$-RoPE (\autoref{aapp:extrap_extra}, \autoref{fig:extrap_yarn_zoom}). The dominant frequency band's logit contribution follows a negative half-cycle of a sinusoid with period $~\approx 2.2\,L_{\text{train}}$, so within the training window it remains negative or near zero.
On relative distances beyond those seen during training, the contribution becomes positive.
To measure the effect of the dominant positional band alone, we measure the performance of \textit{dominant-band interpolation}, which applies positional interpolation only to the dominant positional band.
We find that this prevents the rapid out-of-distribution perplexity explosion, roughly matching the performance of NTK-by-parts (\autoref{tab:extrap_ppl}). For extrapolation evaluations we filter the validation set of C4 for documents longer than 4096 in length.

\paragraph{Compatibility with YaRN.} Standard context-extension methods transfer to LeRoPE directly.
LeRoPE with NTK-by-parts scaling and YaRN's attention temperature outperforms the same recipe applied to RoPE and $p$-RoPE, and also outperforms dominant-band interpolation combined with a temperature.
For each method we sweep over YaRN's temperature hyper-parameter to ensure a fair comparison.  LeRoPE's in-distribution advantage therefore carries over under the extension tooling practitioners already use; the best configuration for extrapolation overall
is LeRoPE with YaRN (\autoref{tab:extrap_ppl}).
See \appref{aapp:extrap_extra} for a description of NTK-by-parts and YaRN.

\textbf{Needle-in-a-Haystack.} We evaluate the 2.52B models on single-needle NIAH with YaRN at sequence lengths up to 32k.
Single-needle retrieval is saturated within $L_{\text{train}}$ (\S\ref{ssec:niah}), so accuracy beyond it isolates where retrieval breaks down with length.
Averaged across needle depths, LeRoPE outperforms both RoPE and $p$-RoPE at every evaluated length (\autoref{fig:32k_niah}).

\begin{figure}[t]
  \centering
  \includegraphics[width=0.98\linewidth]{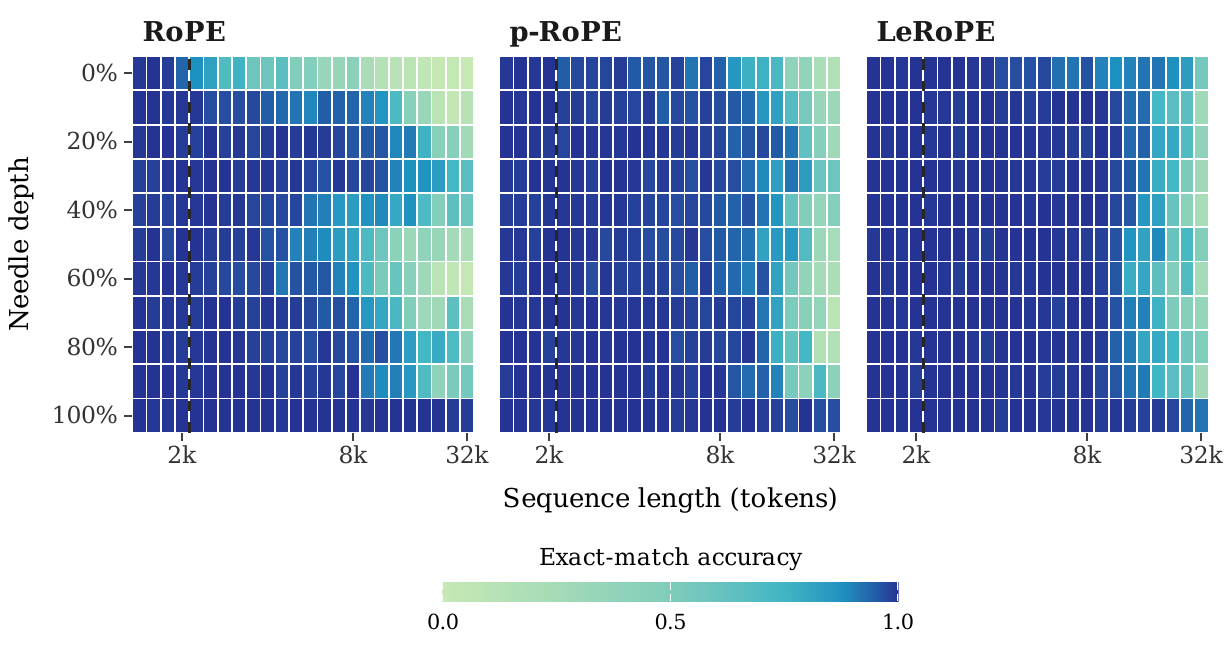}
  \caption{Single-needle NIAH accuracy with NTK-by-parts + YaRN for RoPE, $p$-RoPE, and LeRoPE on sequences up to 32k. We show that LeRoPE generally outperforms both RoPE and $p$-RoPE, especially at longer sequence lengths and for keys early in the context.}
  \label{fig:32k_niah}
\end{figure}

\section{Conclusion}

We propose LeRoPE, which makes RoPE's frequencies learnable through one scale per frequency band. Across a scaling ladder from 52M to 2.5B parameters, LeRoPE outperforms RoPE and partial RoPE at every scale, with RoPE requiring 3.4\% more compute to match it at the largest scale. Given its consistent performance improvements on in-distribution language modeling, future work may explore the merits of learned RoPE parameters in local layers of local-global models~\citep{Kamath2025Gemma3T}.

\section{Limitations}

We only train models up to 2.5B parameters, and our pretraining corpus is limited to C4, with small experiments on code. Further, while Chinchilla-optimal token budgets may be sound for scaling experiments, these models are severely under-trained compared to current language models.
Though cross-seed variance is measured for small model scales, runs about 608M remain trained on one seed and may contain different cross-seed variance.
The gains of LeRoPE on overtrained models at larger scales remain to be shown.

\bibliographystyle{plainnat} %
\bibliography{cites}

\newpage
\appendix
\part{}
\thispagestyle{empty}
{%
\bfseries\Large\vspace{-2.0em}%
\begin{center}
\makeatletter
\@toptitlebar
Supplementary Materials\\[0.25em]%
Learnable RoPE Frequencies Improve Language Modeling%
\@bottomtitlebar
\makeatother
\end{center}%
}
\renewcommand\ptctitle{}
\mtcsetfeature{parttoc}{open}{}
\setlength{\ptcindent}{0pt}
\noptcrule
\parttoc[c]
\clearpage
\suppressfloats[t]

\section{Ladder Notes}
\label{app:ladder_notes}
\subsection{Architecture}
\label{ssec:arch}

\begin{table}[t]
\centering
\caption{Optimal peak learning rates per scale. All three
positional-encoding methods share the same optimum at every swept scale;
for the two largest models the grid index is extrapolated from the smaller
scales, across which the optimum tracks $\eta \propto 1/N_{\text{tot}}$.}
\label{tab:lr_optima}
\begin{tabular}{lcl}
\toprule
Scale & Peak LR $\eta$ & Determination \\
\midrule
52M   & $1.10\times10^{-2}$ & swept \\
217M  & $2.76\times10^{-3}$ & swept \\
608M  & $9.77\times10^{-4}$ & swept \\
1.34B & $4.88\times10^{-4}$ & extrapolated \\
2.52B & $3.45\times10^{-4}$ & extrapolated \\
\bottomrule
\end{tabular}
\end{table}

We use the Enoki scaling ladder of \citet{Charles2025CommunicationEfficientLM, Ferbach2026LogarithmictimeSF}: the head dimension is fixed at 64 and depth and width are co-scaled, yielding models on the shallower, wider end. 
The architecture uses RoPE, pre-norm, QK-norm applied \textit{before} RoPE, no biases, and no weight tying. \autoref{table:enoki_sizes} lists the model configurations.

We deviate from the Enoki recipe as follows:
\begin{itemize}
  \item \textbf{Data and tokenizer:} C4~\citep{Raffel2019ExploringTL} with a
  32{,}000-token SentencePiece vocabulary, rather than FineWeb with the
  50{,}304-token GPT-2 tokenizer.
  \item \textbf{Document masking:} We mask attention across document boundaries within packed sequences; prior Enoki work does not state whether it does so.
  \item \textbf{Normalization:} RMSNorm in place of LayerNorm for the
  pre-norms, QK-norm, and final norm.
  \item \textbf{Learning-rate schedule:} cosine decay to zero rather than to
  $0.1\times$ peak, with a warmup of 1000 steps rather than $2\%$ of total
  steps.
  \item \textbf{Optimizer:} AdamW with $\beta_2 = 0.95$ rather than $0.999$
  ($\beta_1 = 0.9$ in both), and independent weight decay $\lambda = 8/T$,
  where $T$ is the total number of training steps (with weight decay of $\lambda=\omega/T$, \citeauthor{Ferbach2026LogarithmictimeSF} report $\omega = 4$ to be robust; both $\omega=4$ and $\omega=8$ lie within their swept range).
  \item \textbf{Batch size:} 64 sequences (global batch size of $131{,}072$ tokens)
  instead of 32, in line with the batch-size findings of
  \citet{Charles2025CommunicationEfficientLM} for this ladder.
  \item \textbf{Gradient clipping:} global-norm $1.0$ instead of $0.5$.
\end{itemize}

All weights are initialized with fan-in scaling,
$\text{std} = 1/\sqrt{\text{fan-in}}$; the MLP and attention output
projections receive an additional depth-dependent factor,
$\text{std} = 1/\sqrt{2\, n_{\text{layer}} \cdot \text{fan-in}}$; and token
embeddings use $\text{std} = 1/\sqrt{d_{\text{embd}}}$, giving each
embedding vector unit expected norm. 
See \autoref{table:training_config_main} for the complete set of hyper-parameters.

\renewcommand{\arraystretch}{1.1}
\begin{table}[t]
\centering
\caption{Architectural details of our scaling ladder following
\citet{Ferbach2026LogarithmictimeSF}. Head dimension is fixed to 64 across scale, head count is $\tfrac{4}{3}n_{\text{layer}}$, width is $d_{\text{embd}} = 64
\times \text{heads}$. We use a vocabulary of size $32{,}000$ and include the non-embedding ($N_{\text{ne}}$) and total ($N_{\text{ne}}$) parameter counts.}
\label{table:enoki_sizes}
\begin{tabular}{ccccc}
\toprule
Layers & Embd Dim & MLP Hidden &
$N_{\text{ne}}$ & $N_{\text{tot}}$ \\
\midrule
6    & 512   & 2,048  & 18.87M  & 51.64M  \\
12   & 1,024 & 4,096  & 150.99M & 216.53M \\
18   & 1,536 & 6,144  & 509.61M & 607.91M \\
24   & 2,048 & 8,192  & 1.21B   & 1.34B   \\
30   & 2,560 & 10,240 & 2.36B   & 2.52B   \\
\bottomrule
\end{tabular}
\end{table}
\setlength{\textfloatsep}{8pt} 
\renewcommand{\arraystretch}{.9}
{\footnotesize
\ctable[
caption={Training Hyperparameters},
label={table:training_config_main},
pos=t!
]{l | c}{\tnote[1]{$T$ denotes total training steps; weight decay omitted on LeRoPE parameters, embedding parameters, and normalization layers.}}
{
\toprule
\textbf{Parameter} & \textbf{Value} \\
\midrule
Sequence length & 2048 \\
Batch size (sequences) & 64 \\
Tokens per batch & 131,072 \\
Vocabulary size & 32,000 \\
Optimizer & AdamW ($\beta_1{=}0.9$, $\beta_2{=}0.95$) \\
Adam $\epsilon$ & 1e-8 \\
Warmup window & 1000 steps \\
LR rule, $\gamma(t)$ & cosine decay \\
Final LR & 0 \\
Precision/Optimizer state precision & bfloat16/float32 \\
Gradient clipping & 1.0 (global norm) \\
Weight-decay\tmark[1], $\lambda$ & $8/T$ (independent) \\
\bottomrule
}
}

\subsection{Compute Multiplier Estimation}
\label{aapp:cm_calc}

\paragraph{Choice of interpolation axis.}
We use the piecewise-linear baseline of \citet{Ferbach2026LogarithmictimeSF},
but interpolate in $\log C$ rather than raw $C$. For compute budgets outside
the ladder, we continue the nearest line segment.
Loss--compute curves are convex. Linear interpolation in $C$ lies above the
true curve between ladder points, so inverting this baseline overestimates the
compute budget and the multiplier. With nine scales and adjacent budgets
within a factor of ${\sim}2$--$7$, the effect is small on their ladder.
Across our decade-wide gaps, it is severe in $C$ but smaller in $\log C$,
where the power-law curve is far straighter. At the largest scale, LeRoPE's
loss is below that of every RoPE model. Estimating the multiplier requires
extrapolating the RoPE curve, so the top-scale values are conservative.

\paragraph{Parametric alternative: efficiency gain.}
The more common estimator fits a scaling law to the baseline
ladder~\citep{kaplan2020scalinglawsneurallanguage, hoffmann2022trainingcomputeoptimallargelanguage, mai-thinking},
\begin{equation}
    L = f(C) = A C^{-\alpha} + E,
\end{equation}
and reports the efficiency gain $\mathrm{EG} = f^{-1}(L')/C'$ for a variant
reaching loss $L'$ at cost $C'$. Fitting the RoPE ladder gives $A = 26.4$,
$\alpha = 0.0574$, and $E = 0.641$. The fitted $\alpha$ is within the range
reported for fits with an irreducible term by
\citet[Table~15]{Ferbach2026LogarithmictimeSF}.
\autoref{tab:eg_vs_cm} compares the efficiency gains with the piecewise
multipliers of \autoref{ssec:compute_multipliers}, which differ by at most
$0.009$ at every scale.

\begin{table}[h]
\centering
\caption{LeRoPE compute multipliers under the two estimators.}
\label{tab:eg_vs_cm}
\begin{tabular}{lccccc}
\toprule
 & 52M & 217M & 608M & 1.34B & 2.52B \\
\midrule
Piecewise $\log C$ (ours) & 1.071 & 1.048 & 1.047 & 1.032 & 1.034 \\
Efficiency gain (fitted)  & 1.064 & 1.047 & 1.039 & 1.039 & 1.031 \\
\bottomrule
\end{tabular}
\end{table}

\paragraph{Fit instability in compute multiplier estimates.}
With five ladder points and three parameters, the irreducible loss $E$ is
weakly identified. Every $E \in [0.3, 1.0]$ fits the ladder with RMSE below
$0.006$ nats (\autoref{tab:eg_sensitivity}). The RMSE differences between
these fits are comparable to per-run seed variability (\autoref{tab:seeds}),
so run-to-run noise can move the best-fit $E$ across this range. A
misspecified fit can have residuals larger than the $0.003$--$0.010$-nat gaps
being measured, whereas the piecewise estimator has no residual at measured
scales. \citet{Ferbach2026LogarithmictimeSF} report the same identifiability
issue. Their fitted single-power-law exponent ranges from $0.048$ to $0.074$
as the assumed saturation level changes. Our ladder is sparser and our effect
is smaller, so the problem is more severe.

\begin{table}[h]
\centering
\caption{Sensitivity of the fitted efficiency gain to the assumed
irreducible loss $E$. Every $E \in [0.3, 1.0]$ fits the RoPE ladder with
RMSE below $0.006$ nats, yet the implied per-scale gains vary by up to
$0.14$ and include impossible values below $1$ (bold), despite LeRoPE
reaching strictly lower loss at every scale.}
\label{tab:eg_sensitivity}
\begin{tabular}{lcc ccccc}
\toprule
 & & & \multicolumn{5}{c}{EG per scale} \\
\cmidrule(lr){4-8}
$E$ & $\alpha$ & RMSE (nats) & 52M & 217M & 608M & 1.34B & 2.52B \\
\midrule
0.0   & 0.044 & $6.4\mathrm{e}{-3}$ & 1.028 & 1.118 & 1.090 & 1.031 & \textbf{0.952} \\
0.3   & 0.050 & $3.8\mathrm{e}{-3}$ & 1.043 & 1.089 & 1.069 & 1.034 & \textbf{0.982} \\
0.5   & 0.054 & $1.8\mathrm{e}{-3}$ & 1.055 & 1.065 & 1.052 & 1.036 & 1.008 \\
0.641 & 0.057 & $5.3\mathrm{e}{-4}$ & 1.064 & 1.047 & 1.039 & 1.039 & 1.031 \\
0.8   & 0.062 & $2.2\mathrm{e}{-3}$ & 1.075 & 1.024 & 1.022 & 1.043 & 1.063 \\
1.0   & 0.069 & $5.5\mathrm{e}{-3}$ & 1.090 & \textbf{0.989} & \textbf{0.997} & 1.051 & 1.118 \\
\bottomrule
\end{tabular}
\end{table}

\paragraph{Sharing the saturation across methods.}
Following \citet{Ferbach2026LogarithmictimeSF}, we fit RoPE, $p$-RoPE, and
LeRoPE jointly with a shared $E$ and per-method $(A, \alpha)$. The best fit
($E = 0.643$,
joint RMSE $4.7\times 10^{-4}$ nats) leaves the efficiency gains unchanged
to within $0.001$. The identifiability window does not narrow; every shared
$E \in [0.3, 1.0]$ still fits within $0.006$ nats. Because the three
curves are nearly parallel, adding methods provides little information about
the floor. The degeneracy comes from the range of scales, not the number of
points. The fitted exponents agree to within $0.5\%$ ($\alpha \in [0.0573,
0.0576]$), similarly to \citet{Ferbach2026LogarithmictimeSF}.

\paragraph{Richer functional forms.}
\citet{Ferbach2026LogarithmictimeSF} fit a double power law
$L = a + bC^{-c} + eC^{-f}$ across five optimizer curves with a shared
saturation $a$. It is identifiable with ${\sim}10$ scales per curve, but our
five-point ladder has no residual degrees of freedom. A hinge-type
broken power law~\citep{caballero2023brokenneuralscalinglaws} fits our ladder \emph{worse} than
the saturating form (RMSE $1.1\times 10^{-3}$ vs.\ $5.3\times 10^{-4}$
nats), recovers nearly equal exponents on both sides of a weakly identified
break, and exhibits the same efficiency-gain instability, including values
below $1$. The ladder favors saturation over a break.

We therefore report efficiency gain as corroboration and keep the piecewise
estimator primary.

\subsection{Seed Variability}
\label{ssec:seeds}

Seeds control model initialization and data ordering. To bound their
effect, we repeat the 217M runs (as well as RoPE with two additional bases) and the 608M runs under three matched seeds (Table~\ref{tab:seeds}).
All comparisons are within a seed: the two runs share initialization draws and data order, so their loss fluctuations are strongly correlated and the paired margins vary far less than the marginal deviations ($s_d = 0.0019$ vs.\ $\sigma$ up to $0.0020$ at 217M; 5$s_d = 0.0010$ at 608M). LeRoPE improves on every baseline under every 5 individual seeds at both scales, and per-method variability decreases with scale ($\sigma_{\text{LeRoPE}}$: $0.0019 \to 0.0004$), supporting the reading that the unseeded larger-scale gaps in Table~\ref{tab:ppl} also exceed noise.

\newpage
\section{Gradient Dynamics}
\label{app:grad_dynamics}

\theoremstyle{plain}

For clarity in this section, we will omit the pre-softmax scaling factor $1/\sqrt{d}$, operate on a single head, and take $o_i$ to be the post-softmax accumulation of values before output projection $W_o$.

\subsection{Definitions and Lemmas}

We begin with the following definitions:

\newcommand{\Imag}[1]{\text{Im}\left(#1\right)}
\newcommand{\Real}[1]{\text{Re}\left(#1\right)}
\newcommand{\E}{\mathbb{E}}

\begin{align*}
\label{eq:definitions}
    g_s &\triangleq \frac{\partial \LLL}{\partial o_s} \\
    S_{st}^{(m)} &\triangleq (s-t)\|q_s^{(m)}\|\|k_t^{(m)}\| \sin \left( \phi_{st}^{(m)} + (s-t)\theta^{(m)} \right) \\
    C_m(\delta) &\triangleq \sum_{\substack{s,t \\s-t=\delta}} \alpha_{st}   \|q_s^{(m)} \| \|k_t^{(m)}\|
    e^{i\phi_{st}^{(m)}}
    (v_t - o_s)^T g_s \\ 
    P(\theta) &= \Real{
    \sum_{\delta=1}^{L-1} C(\delta) e^{i \delta \theta}
    }
\end{align*}

\begin{lemma}
\label{lem:downstream_gradient}
Let $g_s = \partial \LLL / \partial o_s$ be the gradient of the loss with respect to output vector $o_s$. Let $\partial \LLL / \partial D_{st}^{(m)}$ be the gradient of the loss with respect to the dot product produced by band $m$ from query position $s$ to key position $t$.

Then $\partial \LLL / \partial D_{st}^{(m)} = \alpha_{st} g_s^T(v_t - o_s)$.
\end{lemma}

\begin{proof}
\label{proof:downstream_gradient}
First, note that the only effect of the value $D_{st}^{(m)}$ is through its effect on the output vector at position $s$, $o_s$.

Thus, we have:

\begin{align*}
    \frac{\partial \LLL}{\partial D_{st}^{m}} &= \left(\frac{\partial \LLL}{\partial o_s} \right)^T \left(
    \frac{\partial o_s}{\partial D^{(m)}_{st}}
    \right) \\
    &= g_s^T \left(
    \frac{\partial o_s}{\partial D^{(m)}_{st}}
    \right)
\end{align*}

So we only have to find the partial derivative of $o_i$ with respect to the band-$s$ $QK$ product.

The softmax Jacobian gives:

\begin{align*}
    \frac{\partial o_s}{\partial z_{st}} &= \sum_{t' < s} \alpha_{st'}(\mathbbm{1}[t=t'] - \alpha_{st})v_{t'} \\
    &= \alpha_{st}(v_t - o_s) \\
    \frac{\partial z_{st}}{\partial D_{st}^{(m)}} &= 1 \\
    \implies \frac{\partial o_s}{\partial D_{st}^{(m)}} &= \alpha_{st}(v_t - o_s) \\
    \implies \frac{\partial \LLL}{\partial D_{st}^{(m)}} &= \alpha_{st} g_s^T (v_t - o_s)
\end{align*}

\end{proof}

\begin{lemma}
\label{lem:dtft_gradient}
Define $P_m(\theta)$ as the real component of the discrete-time Fourier Transform of the signal $C_m(\delta)$ with frequency $\theta$, as  above.

Then $\frac{\partial \LLL}{\partial \theta^{(m)}} = P_m'(\theta^{(m)})$.
\end{lemma}

\begin{proof}
\label{proof:dtft_gradient}

First, consider the full sum
from \autoref{eq:freq_grad}. We may split the sum by relative offsets $\delta \triangleq s - t$, and for notational compactness define $g_{st}^{(m)} \triangleq \partial \mathcal{L} / \partial D_{st}^{(m)}$

\begin{align*}
\frac{\partial \mathcal{L}}{\partial \theta^{(m)}} &= -
    \sum_{t < s}
    \left(s - t\right)\| q_s \| \|k_t \| \sin \left(
    \phi_{st} + (s - t)\theta^{(m)} \right) \frac{\partial \mathcal{L}}{\partial D_{st}^{(m)}} \\
    &= -\sum_{\delta = 1}^{L-1} \sum_{\substack{s,t \\ s-t = \delta}}
    \delta \, \| q_s \| \| k_{t} \| \sin \left(
    \phi_{st} + \delta \theta^{(m)}
    \right)
    g_{st}^{(m)} \\
    &= -\sum_{\delta = 1}^{L-1} \sum_{\substack{s,t \\ s-t = \delta}}
    \delta \, \| q_s \| \| k_{t} \| \sin \left(
    \phi_{st} + \delta \theta^{(m)}
    \right)
    \alpha_{st} (v_t - o_s)^T g_s
\end{align*}

To split the sine, write it as an exponential and note that all scalar terms outside of the sine are real, allowing us to apply the imaginary operator across the final summed value:

\begin{align*}
    \frac{\partial \mathcal{L}}{\partial \theta^{(m)}} &= -\sum_{\delta = 1}^{L} \sum_{\substack{s,t \\ s-t=\delta}} \delta \Imag{e^{i\phi_{st}} e^{i \delta \theta^{(m)}}} \|q_s \| \|k_t\| \alpha_{st} (v_t - o_s)^T g_s \\
    &= -\Imag{\sum_{\delta = 1}^{L-1} \sum_{\substack{s,t \\s-t=\delta}} \delta e^{i\phi_{st}} e^{i \delta \theta^{(m)}} \|q_s \| \|k_t\| \alpha_{st}(v_t - o_s)^T g_s } \\
    &= -\Imag{\sum_{\delta = 1}^{L-1} 
    \delta e^{i \delta \theta^{(m)}}
    \sum_{\substack{s,t \\s-t=\delta}} e^{i\phi_{st}}  \|q_s \| \|k_t\| \alpha_{st} (v_t - o_s)^T g_s} \\
    &= -\Imag{
        \sum_{\delta = 1}^{L-1} \delta C(\delta) e^{i \delta \theta^{(m)}}
    } \\
    &= -\Imag{ 
        \sum_{\delta = 1}^{L-1} \left(-i \frac{d}{d \theta^{(m)}} \left[ C(\delta) e^{i \delta \theta^{(m)}}\right]\right)
        } \\
    &= \frac{d}{d \theta^{(m)}} \Real{
    \sum_{\delta = 1}^{L - 1}
    C(\delta) e^{i \delta \theta^{(m)}}
    } \\
    &= P'_m(\theta^{(m)})
\end{align*}

\end{proof}

\begin{lemma}
As previously, define $S_{st}^{(m)}$ to be the offset-weighted sine term at RoPE-band $m$; let $v_t$ be the attention value vector at position $s$ and $g_s$ be the downstream loss gradient $\partial \LLL / \partial o_s$. Then the gradient with respect to RoPE parameter $\theta^{(m)}$ is as follows:
$$\frac{\partial \LLL}{{\partial \theta^{(m)}}} = - \sum_{s} \text{Cov}_{t ~\sim p_s(t)} \left(
    v_t^T g_s , S_{st}^{(m)}
    \right)
    $$
\label{lem:cov_gradient}
\end{lemma}

\begin{proof}
\label{proof:cov_gradient}
As in \autoref{lem:dtft_gradient}, note

\begin{align*}
\frac{\partial \LLL}{\partial \theta^{(m)}}
&= -\sum_{s=1}^{s_{max}} \sum_{t = 0}^{s-1}
    \alpha_{st} (s-t) \, \| q_s \| \| k_{t} \| \sin \left(
    \phi_{st} + (s-t) \theta^{(m)}
    \right)
    (v_t - o_s)^T g_s  \\
&= -\sum_{s=1}^{s_{max}} \sum_{t = 0}^{s-1}
    \alpha_{st} S_{st}^{(m)}
    (v_t - o_s)^T g_s  \\
    &= -\sum_{s=1}^{s_{max}} \E_{t ~\sim p_s(t) }  \left[
    S_{st}^{(m)} (v_t - o_s)^T g_s
    \right]
\end{align*}
\end{proof}

Now, note that $\E_{t ~\sim p_s(t)} [(v_t - o_s)^Tg_s] = 0$, as $\E_{t ~\sim p_s(t)} [v_t] = o_s$, and $g_s$ factors out.

Thus, the expected product of the two terms equals their covariance.

So, as desired, we obtain:

\begin{align*}
\frac{\partial \LLL}{\partial \theta^{(m)}}
    &= -\sum_{s=1}^{s_{max}} \text{Cov} \left(S_{st}^{(m)}, v_t^T g_s \right)
\end{align*}

In expectation, with enough samples, in the absence of periodic structure in downstream attention-logit gradients, we would expect this update term to remain centered roughly around zero.

\newpage
\subsection{Phasor Plots}
\label{sec:phasor_plots}

In \autoref{fig:quiver_plots_grad}, we plot the per-offset gradient phasors for all bands, for both LeRoPE and RoPE models.
The final gradient corresponds to the real component of the aggregate phasor.
The dashed line denotes the direction of this aggregate phasor, and the red dot the aggregate phasor, normalized between 0 (full cancellation) and 1 (full coherence).

Gradient (length-weighted) phasors take on similar scales across offsets, as compared to the non-length-weighted objective phasors. This phenomenon is visible in \autoref{fig:phasor_norm_hists}.

\begin{figure}[p]
\centering
\includegraphics[width=\textwidth]{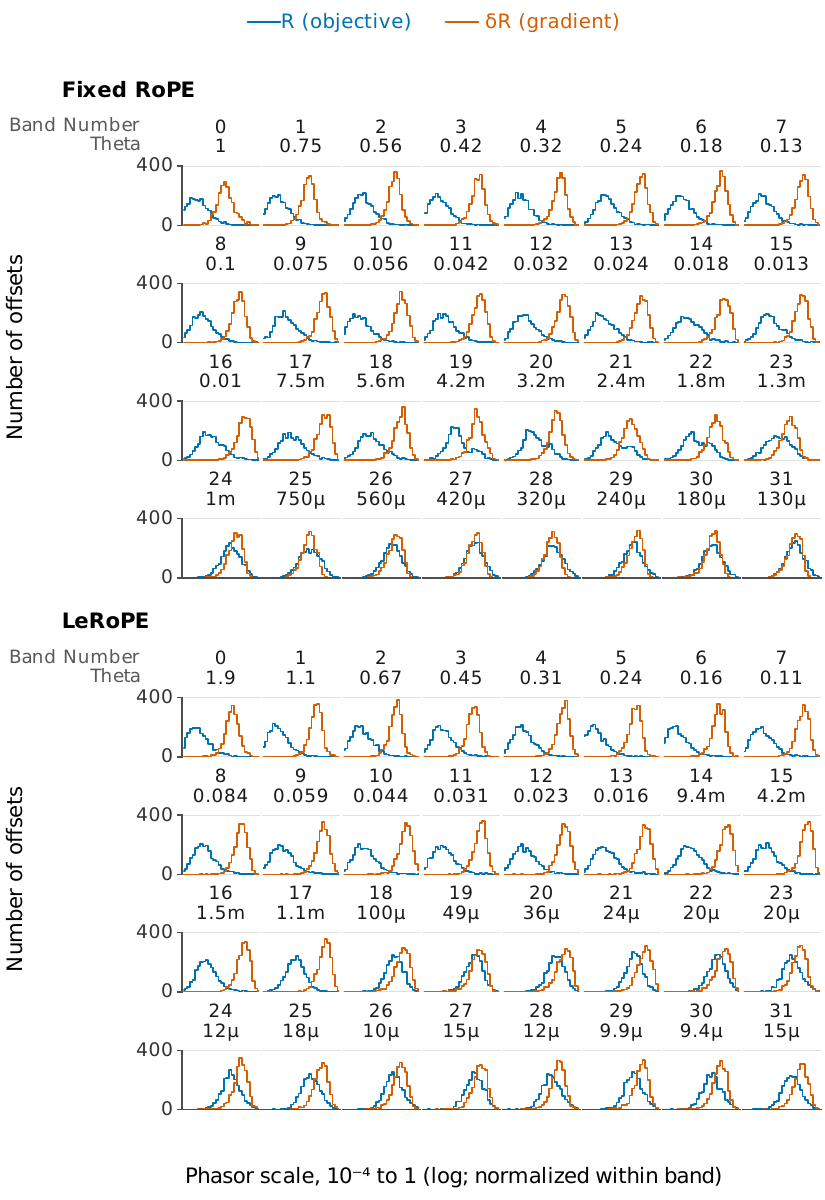}
\caption{Normalized per-offset phasor-scale distributions for fixed RoPE (top) and LeRoPE (bottom) at $L=2048$. Panel headers report the band number and operating $\theta$, with \textup{m}$=10^{-3}$ and $\mu=10^{-6}$. Blue curves weight each offset by $R_\delta$ (objective), while vermillion curves weight by $\delta R_\delta$ (frequency gradient); each distribution is normalized by its within-band maximum. The gradient-weighted distributions move toward one across bands, indicating more even contributions across offsets.}
\label{fig:phasor_norm_hists}
\end{figure}

\begin{figure}[p]
\centering
\includegraphics[width=\textwidth]{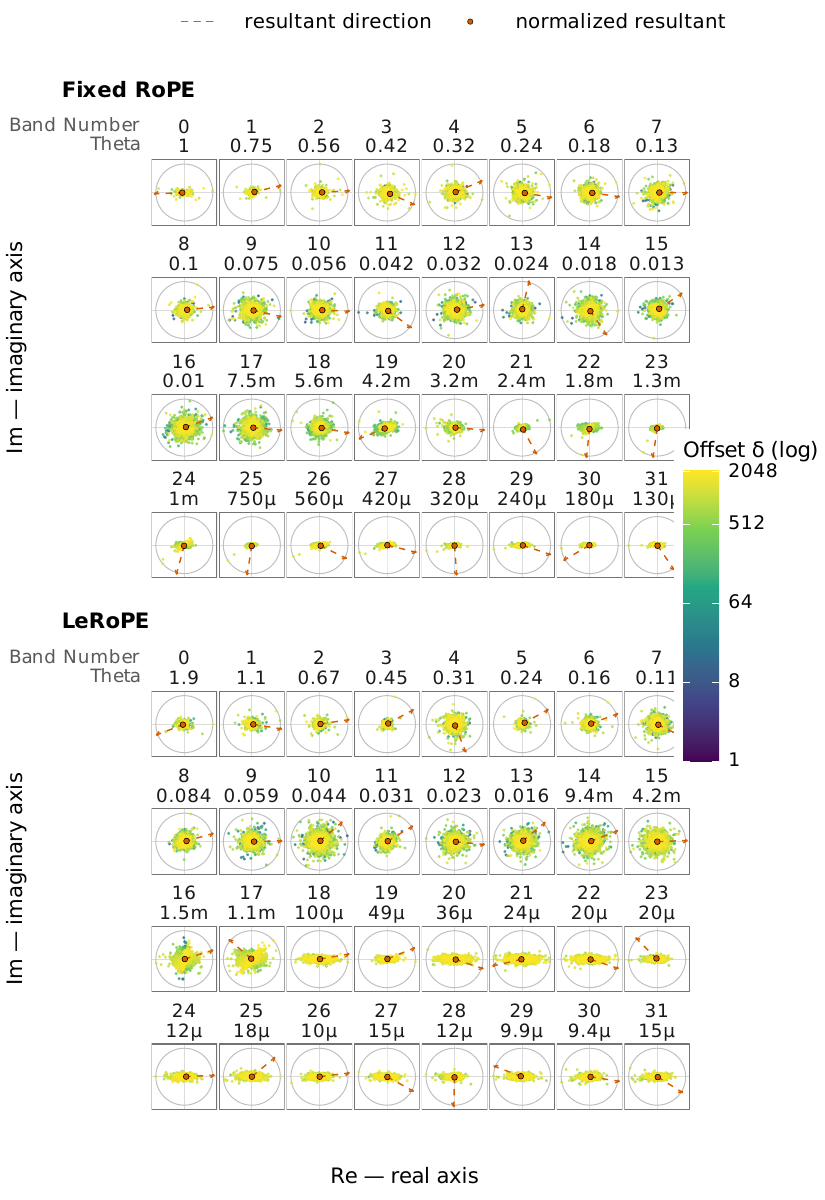}
\caption{Gradient-weighted per-offset phasors for fixed RoPE (top) and LeRoPE (bottom) at $L=2048$. Panel headers report the band number and operating $\theta$, with \textup{m}$=10^{-3}$ and $\mu=10^{-6}$. Radius is normalized within band, color encodes log offset $\delta$, the dashed vermillion arrow gives the aggregate direction, and the outlined dot gives the normalized resultant magnitude. The panels show how cancellation and alignment vary across bands.}
\label{fig:quiver_plots_grad}
\end{figure}

\newpage
\section{Additional Ablations}
\label{app:ablations}

\begin{table}[t]
\centering
\caption{Consolidated ablations at 217M on C4, all under the protocol of
\S\ref{sec:exp_setup} and reported on the shared base seed for
comparability; $\Delta$ is the loss difference to default RoPE (negative is
better). Rows marked $\dagger$ have three matched seeds
(\autoref{tab:seeds}). Differences below ${\sim}0.002$ nats are within
seed variability, so orderings among the three LeRoPE parametrizations, and
between $p \in \{0.5, 0.75\}$, are not meaningful. Fixed LeRoPE trains with
frequencies frozen to those of an independent LeRoPE run at the same scale
(cf. \autoref{tab:ppl_warmstart}).}
\label{tab:ablations}
\begin{tabular}{lcc}
\toprule
Variant & Val loss (nats) & $\Delta$ vs RoPE \\
\midrule
RoPE, $b = 2048$$^\dagger$                & 2.9107          & $+0.0073$ \\
RoPE, $b = 10{,}000$ (default)$^\dagger$  & 2.9035          & ---       \\
RoPE, $b = 500{,}000$$^\dagger$           & 2.8996          & $-0.0039$ \\
\midrule
$p$-RoPE, $p = 0.25$                      & 2.9208          & $+0.0173$ \\
$p$-RoPE, $p = 0.5$                       & 2.9029          & $-0.0006$ \\ %
$p$-RoPE, $p = 0.75$ (baseline)           & 2.9029          & $-0.0006$ \\
\midrule
LeRoPE (log scalar, shared)$^\dagger$     & \textbf{2.8978} & $\mathbf{-0.0057}$ \\
\quad linear scalar                       & 2.8978          & $-0.0057$ \\
\quad direct frequencies                  & 2.8980          & $-0.0055$ \\
\quad per layer and head                  & 2.8996          & $-0.0039$ \\
\quad frozen transferred freqs.\ (Fixed LeRoPE) & 2.8999    & $-0.0036$ \\
\bottomrule
\end{tabular}
\end{table}

In this section we include additional ablations that we ran during the process of testing LeRoPE.

\textbf{Parametrization.} LeRoPE learns the per-band scale in log space,
$\hat\theta_m = e^{\alpha_m}\theta_m$. We compare against a linear scalar,
$\hat\theta_m = \alpha_m \theta_m$, and against directly learning each frequency, $\hat\theta_m = \alpha_m$. Both recover RoPE at initialization. 

All three reach the same loss to within $0.0003$ nats (\autoref{tab:ablations}), within seed variability, suggesting the choice is not critical with Adam's per-parameter step normalization. However, we observed very unstable trajectories for the directly learned frequencies and aliasing around 0 for the linear scalars. We retain the log parametrization because it removes these issues.

\textbf{Finer Learned Granularities.} LeRoPE shares one set of scales across all layers and heads. Learning a separate set per layer and head adds parameters but no benefit. It reached $2.8996$ nats, $0.0018$ behind shared LeRoPE which is within one seed standard deviation (Table~\ref{tab:ablations}).

\textbf{Change of RoPE Base.} Prior work finds $b = 10{,}000$ can be suboptimal, with gains reported from both raising and lowering the base~\citep{Liu2023ScalingLO}. We train RoPE with $b = L_{\text{train}} = 2048$, following FMRoPE~\citep{oka2026frequency}, and $b = 500{,}000$, the value adopted for long-context models~\citep{xiong2023effectivelongcontextscalingfoundation} (\autoref{tab:seeds}). Matching the base to the training length degrades in-window loss on every seed ($+0.0069$ nats), consistent with the in-window side of the interpolation--extrapolation trade-off; raising it to $500{,}000$ improves loss on every seed ($-0.0038$ nats), capturing roughly $60\%$ of LeRoPE's gain. LeRoPE outperforms every tested base under every seed (\autoref{tab:seeds}).

\textbf{$p$-RoPE fraction.} Our baseline uses $p = 0.75$, the best-performing value in \citet{Barbero2024RoundAR}. We swept $p \in \{ 0.25, 0.5, 0.75\}$ at 217M. $p = 0.5$ matches $p = 0.75$ (both within seed variability of RoPE).
$p = 0.25$ leaves no band rotating with a period beyond ${\sim}50$ tokens and degrades substantially ($+0.0173$ nats, ${\sim}9\times$ seed variability). Notably, LeRoPE's learned suppression point at $L_{\text{train}} = 2048$ sits near the same boundary (band ${\sim}17$ of 32, an effective $p \approx 0.5$), while its gain comes from reshaping the bands it retains, which no choice of $p$ can express.

\begin{table}[t]
\centering
\setlength{\tabcolsep}{5pt}
\caption{C4 validation loss (nats) across three matched random seeds; seed
42 is the run reported in \autoref{tab:ppl}. $\bar\Delta$ is each
baseline's mean paired margin behind LeRoPE, which is positive under every
individual seed ($s_d = 0.0019$ vs.\ RoPE and $0.0008$ vs.\ $b{=}500$k at
217M; $0.0010$ vs.\ RoPE at 608M). Seed effects are strongly correlated
across methods trained on the same seed, so paired margins vary
substantially less than the marginal deviations suggest. Per-method seed
variability decreases with scale, and at 608M the two methods' ranges do
not overlap.}
\label{tab:seeds}
\begin{tabular}{lccccc}
\toprule
 & Seed 42 & Seed 43 & Seed 44 & mean $\pm$ std & $\bar\Delta$ \\
\midrule
\multicolumn{6}{l}{\emph{217M}} \\
RoPE ($b{=}2048$)    & 2.9107 & 2.9100 & 2.9137 & $2.9115{\pm}0.0020$ & 0.0134 \\
RoPE ($b{=}10$k)     & 2.9035 & 2.9050 & 2.9053 & $2.9046{\pm}0.0010$ & 0.0066 \\
RoPE ($b{=}500$k)    & 2.8996 & 2.8997 & 2.9031 & $2.9008{\pm}0.0020$ & 0.0027 \\
LeRoPE               & \textbf{2.8978} & \textbf{2.8963} & \textbf{2.9001}
                     & $\mathbf{2.8981{\pm}0.0019}$ & --- \\
\midrule
\multicolumn{6}{l}{\emph{608M}} \\
RoPE ($b{=}10$k)     & 2.6304 & 2.6290 & 2.6288 & $2.6294{\pm}0.0009$ & 0.0040 \\
LeRoPE               & \textbf{2.6253} & \textbf{2.6250} & \textbf{2.6257}
                     & $\mathbf{2.6254{\pm}0.0004}$ & --- \\
\bottomrule
\end{tabular}
\end{table}

\renewcommand{\sectionautorefname}{Appendix}
\renewcommand{\subsectionautorefname}{Appendix}

\newpage

\newpage
\section{Extrapolation and Dominant Band Wavelengths}
\label{app:freq_plots}

\subsection{Dominant Positional Band Wavelength}

\begin{table}[ht]
\centering
\caption{Dominant-band indices and learned wavelengths for different models and training context lengths.}
\label{tab:dom_band_stats}
\begin{tabular}{lcccccc}
\toprule
Scale (params) & $L$ & Dominant Band & $e^\alpha$ & $\hat{\theta}^{(m)}$ & $\lambda_{dominant}$ (tok) & $\hat{\theta}^{(m)} L$ \\
\midrule
217M & 2048 & 16 & 0.1481 & 0.001481 & 4243.3 & 3.033 \\
217M  & 4096 & 18 & 0.1178 & 0.000663 & 9481.7 & 2.714 \\
217M  & 8192 & 19 & 0.0823 & 0.000347 & 18097.8 & 2.844 \\
608M & 2048 & 17 & 0.1945 & 0.001459 & 4307.8 & 2.987 \\
1.34B & 2048 & 17 & 0.1955 & 0.001466 & 4285.6 & 3.003 \\
2.52B & 2048 & 17 & 0.2009 & 0.001507 & 4170.3 & 3.086 \\
\bottomrule
\end{tabular}
\end{table}

\paragraph{Leave-one-out band ablation.}
We zero one learned frequency band at a time and measure the increase in
evaluation loss relative to the baseline. Zeroing band 17 increases loss by
$0.762$ nats. The next-largest loss increase comes from zeroing band 26,
which increases loss by only $0.069$ nats.

\begin{figure}[h]
  \centering
  \includegraphics[width=0.95\linewidth]{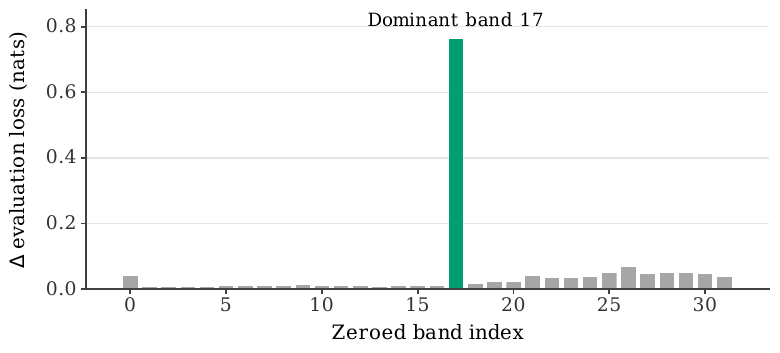}
  \caption{Leave-one-out importance of learned frequency bands. Each bar shows
  the loss increase from zeroing the indicated band, relative to a baseline
  loss of $2.309$.
  Zeroing out Band 17 increases evaluation loss by $0.762$ nats, far exceeding others.}
  \label{fig:leave_one_out_bands}
\end{figure}

\subsection{Extrapolation Details}
\label{aapp:extrap_extra}

\paragraph{NTK-by-parts.} NTK-by-parts~\citep{Peng2023YaRNEC} sets each
band's interpolation from its number of rotations during training. For band
$m$ with frequency $\varphi_m$, let $r_m = L_{\text{train}}\varphi_m / 2\pi$
be its rotation count during training. In LeRoPE, $\varphi_m=\hat\theta_m$.
Let $s = \min(1, L_{\text{train}}/L_{\text{doc}})$ be the per-document
position interpolation factor. We multiply each frequency by
$M_m = s + w_m(1-s)$, where $w_m = \mathrm{clip}\big((r_m-r_{\text{lo}})/
(r_{\text{hi}}-r_{\text{lo}}),\,0,\,1\big)$, with $r_{\text{lo}}=1$ and
$r_{\text{hi}}=32$~\citep{Peng2023YaRNEC}. Bands with fewer than $r_{\text{lo}}$
rotations are scaled by $s$, while bands with more than $r_{\text{hi}}$
rotations are unchanged. Between these thresholds, the multiplier increases
linearly. We apply interpolation per document rather than at a fixed target
length and compute rotation counts from each model's own frequencies.
LeRoPE's dominant band, slowed to ${\sim}0.5$ rotations per window, is always fully
interpolated.

\paragraph{YaRN.} YaRN~\citep{Peng2023YaRNEC} adds an attention temperature
to NTK-by-parts to keep attention from flattening at longer contexts. With
per-document extension factor $s_{\text{ext}} = \max(1, L_{\text{doc}}/L_{\text{train}})$,
attention logits are multiplied by $(c \ln s_{\text{ext}} + 1)^2$ before the
softmax.
We sweep YaRN's temperature coefficient $c$ for each method and use $c=0.1$,
matching \citet{Peng2023YaRNEC}. Both mechanisms operate at inference only;
no parameters are updated.

\begin{table}[h]
  \centering
  \caption{Average perplexity for 2.52B models under NTK-by-parts + YaRN as well as LeRoPE with only interpolating the dominant band and extrapolating the rest. Split by in-distribution ($<L_{\text{train}}$) and extrapolated ($\geq L_{\text{train}}$) positions over the 0--4095 window. LeRoPE is lowest in every column.}
  \label{tab:extrap_ppl}
  \begin{tabular}{lccc}
    \toprule
    Method & In-dist & Extrap & Full \\
    \midrule
    RoPE     & 10.138 & 8.861 & 10.042 \\
    $p$-RoPE & 10.133 & 8.838 & 10.035 \\
    \midrule  
    LeRoPE \textit{(dom. band only)} & 10.112 & 8.965 & 10.026 \\
    LeRoPE   & \textbf{10.106} & \textbf{8.776} & \textbf{10.005} \\
    \bottomrule
  \end{tabular}
\end{table}
\FloatBarrier

To complement these sequence-level averages, \autoref{fig:extrap_yarn_zoom}
shows per-token loss over the full evaluation range and a zoomed view of the
extrapolation region.

\begin{figure}[h]
  \centering
  \includegraphics[width=0.98\linewidth]{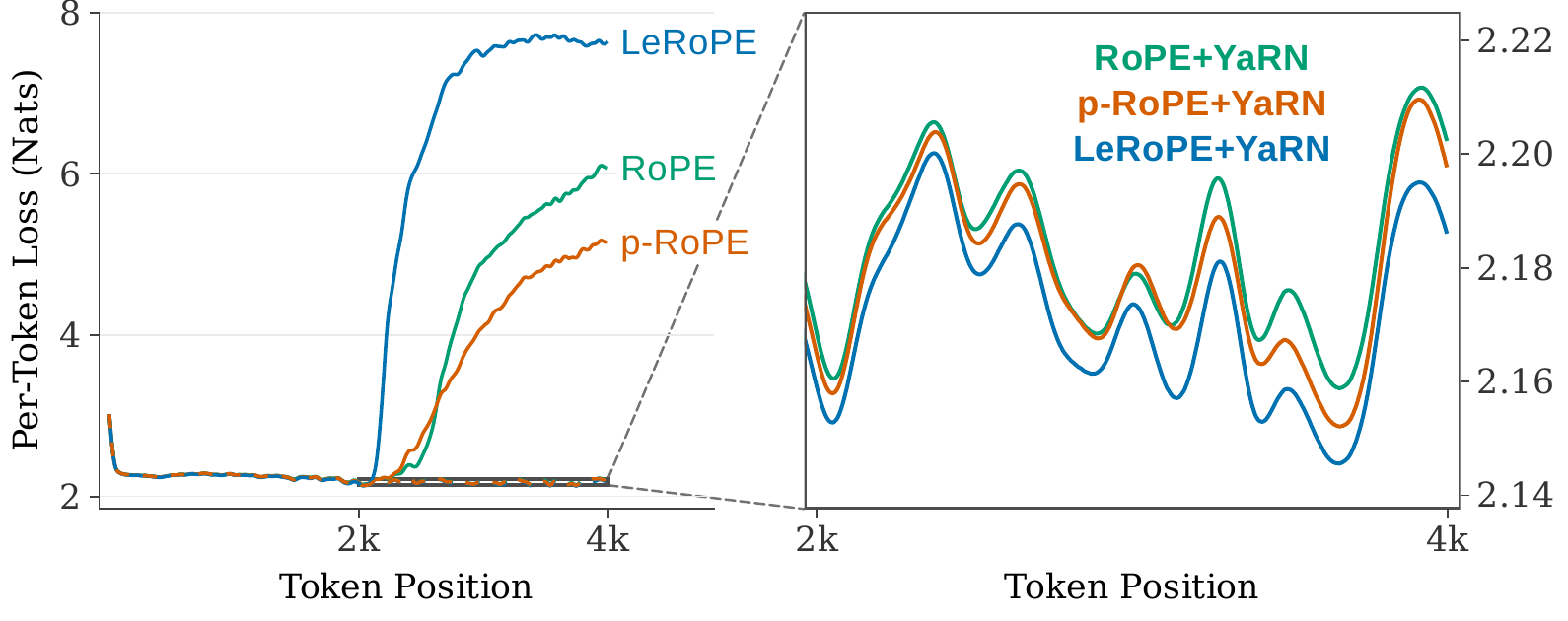}
  \caption{Per-token loss under naive extrapolation (solid) and NTK-by-parts + YaRN (dashed), evaluated up to $2\times L_{\text{train}}$. Left: Full-range results smoothed with a Gaussian kernel, $\sigma = 20$. Right: YaRN results over the extrapolation region smoothed with $\sigma = 60$.}
  \label{fig:extrap_yarn_zoom}
\end{figure}

\newpage
\section{Code Pretraining}
\label{app:code}

We evaluate LeRoPE on the Python split of the StarCoder data~\citep{li2023starcoder} to test whether its gains and frequency patterns in Section~\ref{sec:learned_freqs} are specific to C4 or generalize to code. We pre-train 217M-parameter models with the same setup as Section~\ref{sec:exp_setup}, using a 32{,}000-token SentencePiece tokenizer trained on code to match vocabulary size and parameter count.

\paragraph{Language modeling.} At 217M, LeRoPE improves over RoPE by $0.0104$ nats (\autoref{tab:code}). By contrast, $p$-RoPE improves by $0.0006$ nats, The StarCoder-Python runs use a single seed. If we assume run-to-run variability is comparable to C4 (\autoref{tab:seeds}), LeRoPE's $0.0104$-nat loss reduction exceeds that variability several times over.

\begin{table}[h]
\centering
\caption{Validation loss on StarCoder-Python at 217M (base seed), under the
training recipe of Section~\ref{sec:exp_setup}.}
\label{tab:code}
\begin{tabular}{lcc}
\toprule
Method & Val loss (nats) & $\Delta$ vs RoPE \\
\midrule
RoPE     & 1.2501 & --- \\
$p$-RoPE & 1.2495 & -0.0006 \\
LeRoPE   & \textbf{1.2398} & -0.0104 \\
\bottomrule
\end{tabular}
\end{table}

\paragraph{Learned frequencies.} We analyze 217M-parameter code models
trained at $L_{\text{train}} \in \{2048, 4096, 8192\}$ following the
procedure in Section~\ref{sec:learned_freqs}. On StarCoder, fast bands remain
close to their fixed-RoPE frequencies while slow bands learn wavelengths
substantially longer than their fixed-RoPE values, as on C4
(\autoref{fig:rope_lerope_parametric}c). As training length increases, learned frequencies deviate from
fixed RoPE at lower frequencies. The dominant wavelength also
increases, from $5{,}250$ tokens at $L_{\text{train}}=2048$ to $16{,}627$
tokens at $L_{\text{train}}=8192$. At $L_{\text{train}}=2048$, a single
positional band dominates (\autoref{fig:code_frequency_analysis}b) with wavelength
$\lambda \approx 5{,}250$ tokens.

\begin{figure}[h]
  \centering
  \begin{subfigure}[b]{0.35\linewidth}
    \centering
    \includegraphics[width=\linewidth]{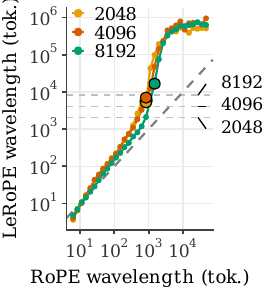}
  \end{subfigure}
  \hfill
  \begin{subfigure}[b]{0.61\linewidth}
    \centering
    \includegraphics[width=\linewidth]{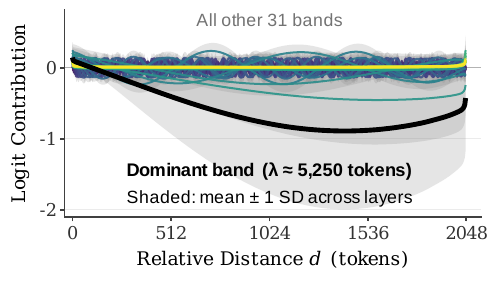}
  \end{subfigure}
  \caption{\textbf{Left}: Learned versus fixed-RoPE wavelength for each band in
  217M LeRoPE models trained on StarCoder-Python at
  $L_{\text{train}} \in \{2048,4096,8192\}$. The diagonal dashed line marks
  $y=x$, horizontal dashed lines mark the training lengths, and outlined points
  mark the dominant bands. \textbf{Right}: Band-level logit contributions as a
  function of raw token offset $d=s-t$ for the $L_{\text{train}}=2048$ model.
  Shading shows $\pm 1$ standard deviation across layers.}
  \label{fig:code_frequency_analysis}
\end{figure}

\end{document}